\documentclass[conference]{IEEEtran}
\IEEEoverridecommandlockouts

\usepackage{cite}
\usepackage{amsmath,amssymb,amsfonts,amsthm}
\usepackage{algorithmic}
\usepackage{graphicx}
\usepackage{booktabs}
\usepackage{textcomp}
\usepackage{xcolor}
\usepackage{url}
\usepackage[ruled,linesnumbered]{algorithm2e}
\def\BibTeX{{\rm B\kern-.05em{\sc i\kern-.025em b}\kern-.08em T\kern-.1667em\lower.7ex\hbox{E}\kern-.125emX}}

\graphicspath{{figs/}}

\newtheorem{theorem}{Theorem}
\newtheorem{proposition}{Proposition}
\newtheorem{corollary}{Corollary}
\theoremstyle{definition}
\newtheorem{definition}{Definition}

\newcommand{\E}{\mathbb{E}}
\newcommand{\Prob}{\mathbb{P}}
\newcommand{\Var}{\mathrm{Var}}
\newcommand{\Cov}{\mathrm{Cov}}
\newcommand{\Corr}{\mathrm{Corr}}
\newcommand{\ind}{\mathbf{1}}
\newcommand{\R}{\mathbb{R}}
\newcommand{\thetahat}{\hat\theta}
\newcommand{\cX}{\mathcal{X}}
\newcommand{\eps}{\varepsilon}
\newcommand{\mc}{\mathrm{MC}}
\newcommand{\cv}{\mathrm{CV}}
\newcommand{\tpr}{\tau}
\newcommand{\fpr}{\phi}
\newcommand{\precn}{\pi}
\newcommand{\Phibar}{\bar\Phi}

\begin{document}

\title{Learned Control Variates for Rare Grid-Outage Estimation at Electric Vehicle Charging Stations}

\author{\IEEEauthorblockN{Sulyman Iyanda}
\IEEEauthorblockA{\textit{Industrial and Systems Engineering Department} \\
\textit{King Fahd University of Petroleum and Minerals}\\
Dhahran, Saudi Arabia\\ sulymaniyanda@gmail.com}
\and
\IEEEauthorblockN{Mansur Arief}
\IEEEauthorblockA{\textit{Industrial and Systems Engineering Department} \\
\textit{King Fahd University of Petroleum and Minerals}\\
Dhahran, Saudi Arabia \\
mansur.arief@kfupm.edu.sa}
}

\maketitle

\begin{abstract}
Public charging of electric vehicles (EVs) adds a large and synchronized load to the distribution feeders. In most networks, it is rare for charging demand, weather, and local grid conditions to overload a feeder at the same time. When it does happen, the consequences can be severe: an overloaded feeder stresses its transformers and can trigger cascading outages across a city. Planners who design for high reliability therefore need to estimate very small outage probabilities. To that end, plain Monte Carlo estimation with an expensive grid simulator can require a prohibitively large number of runs for this task. To reduce the number of runs, we propose a control variate (CV) built from a learned classifier of the outage, trained on cheap observable features such as temperature and load forecasts. We show that for any binary learner, the variance reduction is an explicit function of its recall and false-positive rate. Furthermore, we show that control variate based on a Gaussian proxy of the load loses its effect as outages become rarer, while a learner with non-vanishing precision keeps removing a fixed fraction of the variance even in the extreme rare-event regime. We verify these results in numerical experiments using a feeder-overload simulator across three orders of magnitude of outage probability and obtain that even weak learner with a single feature can remove a quarter of the variance, while a boosted-tree learner is able to remove about 60\%. We then apply the estimator to the feeder network of Surabaya, Indonesia, in which the Gaussian load model leaves every feeder three to seven times the one-day-in-ten-years blackout target. At the same simulation budget, planning the reserve with the learned CV and a 95\% confidence limit reduces the city-wide blackout days by a third compared with Monte Carlo point estimates. These results show that learned control variates can help planners design distribution grids that stay reliable against rare outages.
\end{abstract}


\section{Introduction}
\label{sec:intro}

Electric vehicle (EV) charging is becoming one of the largest new loads on distribution networks. Uncoordinated charging raises the peak demand and could push feeders and transformers beyond their ratings on the few days when charging coincides with a high ambient temperature, a demand forecast miss, or a storm \cite{muratori2018impact,clementnyns2010impact,powell2022charging,nutkani2024impact}. Hosting-capacity studies of distribution networks with EV charging stations show that the binding constraint is the loading on a small number of such extreme days \cite{sun2025hosting}. In cities where the grid already operates with little reserve, one such day could result in a local outage, in which public charging stations are among the first loads to be shed or to lose supply \cite{lukuyu2024managing,khan2024impact}.

Traditionally, the planning literature on charging stations treats an outage as a service failure, in which vehicles arrive, the station is out of service, and revenue is lost \cite{akhra2023robust}. The consequences of a feeder overload, however, go beyond the station. An overloaded feeder that trips its protection moves its load onto the neighbouring feeders, while transformers that run hot age faster. A sequence of such local events is the mechanism through which large blackouts begin \cite{dobson2007complex,hines2009large}. For example, the August 2003 blackout in North America started from a small number of overloaded lines and cost several billion US dollars \cite{taskforce2004blackout}. Hence, for a candidate charging station, the grid operator needs to know the probability that its charging load, together with the other loads on the feeder, causes an overload on a given day. Reliability planning works with such probabilities at levels of $10^{-3}$ to $10^{-6}$ per period. The classical loss-of-load criterion of one day in ten years, for instance, corresponds to $2.7\times10^{-4}$ per day \cite{billinton1996reliability}.

Estimating such a small probability is a rare-event simulation problem. A feeder model with thermal limits, protection logic and stochastic EV arrivals is expensive to run. The Monte Carlo (MC) estimate using $n$ samples with a target probability $\theta$ has relative standard error $\sqrt{(1-\theta)/(n\theta)}$, which means a $10\%$ relative error at $95\%$ confidence requires about $400/\theta$ runs, i.e., four million runs at $\theta=10^{-4}$. Earlier work on robust charging-station design \cite{akhra2023robust} is able to reduce this cost with a control variate (CV) built from a Gaussian proxy of the electricity load. This construction, however, assumes that the modeler already knows a Gaussian variable that is strongly correlated with the outage. In practice however, it is often impractical to obtain such proxy. Furthermore, as we show in this paper, such a proxy loses its effect as outages become rarer.

To that end, we propose in this paper to replace the assumed proxy by a learned one. Cheap features that are observable before the expensive simulation is run, such as the temperature anomaly, the day-ahead load forecast and a storm flag, are often available. A classifier trained on a modest number of simulator runs can then predict, though imperfectly, whether the simulator will report an overload. We show that such a classifier only needs to be a weak learner \cite{schapire1990strength}, whose value as a control variate is determined by its recall and precision only. We also show that, with a suitable training design, the learner keeps its effect at every rarity level. 

In sum, our main contributions are as follows. \textbf{First}, we provide an exact variance formula for the CV estimator built from any binary learner, together with the rare-event limit $1-\mathrm{recall}\times\mathrm{precision}$ of its variance ratio (Theorem~\ref{thm:binary}, Corollary~\ref{cor:rare}). \textbf{Second}, we derive a ceiling on the variance reduction of any learner, which is set by the features, and a proof that a Gaussian proxy loses its effect as the outage becomes rarer (Theorem~\ref{thm:ceiling}, Proposition~\ref{prop:gauss}). \textbf{Third}, we show the exact variance of the estimator with an estimated coefficient and an unlabelled pool, the cost-optimal split between simulator runs and feature draws (Theorem~\ref{thm:pool}), and a training design that keeps the learner informative in the tail (Proposition~\ref{prop:design}). \textbf{Finally}, we demonstrate using a simulation study that compares plain MC, the Gaussian proxy, a weak learner, a strong learner and the Bayes-posterior ceiling at outage probabilities from $10^{-3}$ to $10^{-6}$ based on a case study on the Surabaya network \cite{akhra2023robust}, which quantifies how the load model and the precision of the estimate change the reserve on each feeder and the blackout days the city would experience.

\section{Related Work}
\label{sec:related}

\subsection{EV Charging Load and Grid Outages}
The effect of EV charging on distribution networks has been well documented. Uncoordinated residential charging raises the evening peak and the loading of distribution transformers \cite{muratori2018impact,clementnyns2010impact}. Furthermore, the choice of where and when vehicles charge could change the grid impact of a given fleet by a large factor \cite{powell2022charging}. Reviews of the distribution-level impacts list voltage drops, transformer ageing, losses and protection trips as the main consequences of EV charging, with demand management, storage and network reinforcement as the main remedies \cite{nutkani2024impact,khan2024impact}. Hosting-capacity studies, including recent work which shows that the capacity of a feeder for PV and charging stations is set by the loading on a small number of extreme days \cite{sun2025hosting}. Studies in African cities also show that a grid with little reserve shifts from occasional to frequent interruptions as charging grows \cite{lukuyu2024managing}. The consequences of a feeder overload, however, go beyond the feeder itself. The analysis of cascading failures and of the record of large blackouts in North America shows that a sequence of local overloads is the mechanism through which large blackouts begin \cite{dobson2007complex,hines2009large}, often at a large economic cost \cite{taskforce2004blackout,lacommare2006cost}. For this reason, reliability planning works with small outage probabilities per period \cite{billinton1996reliability}. These studies describe the mechanism and the stakes of an overload. However, they do not address how a utility could estimate a probability of order $10^{-3}$ to $10^{-6}$ for a specific feeder with a simulator that it can only run a limited number of times.

\subsection{Charging-Station Planning under Uncertainty}
Charging-infrastructure planning is usually formulated as a facility-location problem \cite{daskin2013network} that trades investment and operating cost against service quality. Early models were deterministic \cite{frade2011optimal,amilia2022designing}, while later work added stochastic demand \cite{hosseini2015refueling,yildiz2019urban}. A smaller number of studies considered disruptions such as floods \cite{zhang2023optimal} or outages mitigated by on-site storage \cite{hussain2020optimal}. The review in \cite{unterluggauer2022electric} calls for models that quantify randomness explicitly. In the facility-location literature on unreliable facilities \cite{snyder2005reliability}, the failure probabilities are treated as known inputs. Earlier work on robust charging-station design \cite{akhra2023robust} is, to our knowledge, the first to estimate station reliabilities by simulation inside the planning loop, using a Gaussian-proxy control variate to make this affordable. However, such a proxy must be known in advance. Furthermore, the Gaussian load model behind it is exactly the model that, as we show in this paper, misjudges the tail. In this paper, we leave the location model aside and focus on the estimation step, which all of these models take as given.

\subsection{Rare-Event Simulation and Variance Reduction}
Estimating a small probability by plain MC is inefficient, since the relative error grows as $1/\sqrt{n\theta}$ when $\theta$ decreases \cite{asmussen2007stochastic,heidelberger1995fast}. Two classical remedies are importance sampling and control variates. Importance sampling changes the sampling distribution, whereas control variates keep it and subtract a correlated correction with zero mean \cite{kahn1953methods,lavenberg1981perspective,nelson1990control,glasserman2004monte}. Control variates for probability estimation in particular were studied in \cite{hesterberg1998control}. Learned classifiers have also been used to build certifiable importance-sampling proposals for black-box rare events \cite{arief2021deep}. With a learned proposal, however, the quality of the estimate depends on how well the proposal covers the rare-event set, which has to be certified. A learned control variate, in contrast, keeps the estimator unbiased for any classifier. Several related ideas use a cheap model to correct an expensive one. Multi-fidelity MC uses a low-fidelity model as a control for a high-fidelity one and allocates the budget between the fidelities \cite{peherstorfer2016optimal,peherstorfer2018survey}. Prediction-powered inference combines a machine-learning predictor on a large unlabelled set with a small labelled set \cite{angelopoulos2023prediction,angelopoulos2023ppi}. Control functionals build controls with a known mean from the score of the sampling density \cite{oates2017control,south2023regularized}. The estimator in Theorem~\ref{thm:pool}(b) is algebraically the same as the two-fidelity estimator and the prediction-powered estimator. None of these works, however, considers the case in which the quantity of interest is a rare binary event. In particular, they do not show how the variance reduction depends on the classifier, whether it survives as the event becomes rarer, or how the classifier should be trained when a nominal training set contains no event at all.

\subsection{Outage Prediction and Weak Learners}
Data-driven outage prediction models use weather, load and infrastructure features to forecast where and how many outages a storm will cause \cite{guikema2014predicting,nateghi2014power}. Probabilistic forecasts of EV charging load and of ultra-short-term feeder load are also becoming common tools for operational reliability in smart grids \cite{hu2025evload,chen2025reliability}. These models are built and tuned for prediction, with the accuracy, the error or the area under the ROC curve as the criterion. In machine learning, a weak learner is a classifier whose balanced error is only slightly below one half. Boosting shows that such learners can be combined into a strong one \cite{schapire1990strength,kearns1994cryptographic}. In this paper, we use the notion for a different purpose. We show that a weak learner is already a useful control variate, whose value is set by its recall and precision on the rare days (Section~\ref{sec:theory}). These criteria differ from those used for prediction.


\section{Problem Setting and Estimators}
\label{sec:setting}

\subsection{Feeders, Charging Stations and Outages}
We consider a distribution feeder that supplies, among other loads, a public EV charging station. The feeder has a rating, i.e., the thermal limit of its conductors or of the transformer behind it, and a protection scheme that disconnects it when the load exceeds the rating for long enough. On an ordinary day, the peak loading stays well below the rating. However, in a few days,  several effects coincide: a high ambient temperature raises the air-conditioning load and reduces the capacity of the cables, the day-ahead forecast underestimates the demand, a storm shifts load onto the feeder, and many vehicles arrive at the station at once. The loading then reaches the rating and trips. We call such a day an \emph{outage day}.

An outage day has two kinds of cost. For the charging operator, the station is out of service and revenue is lost. For the grid operator, the tripped feeder moves its load onto its neighbors, whose transformers then run hot and age faster. If the neighboring feeders are themselves close to their ratings, the disturbance could propagate as a cascading failure \cite{dobson2007complex}. This second cost is borne by every customer on the affected feeders and is much larger than the first \cite{lacommare2006cost}. Hence, the approval of a charging station depends on the probability of an outage day on its feeder, which the grid operator requires to be small. For example, the criterion of one loss-of-load day in ten years \cite{billinton1996reliability} corresponds to $2.7\times10^{-4}$ per day. In general, the probabilities that a planner needs to certify may lie between $10^{-3}$ and $10^{-6}$.

The utility has three sources of information about a feeder. \textbf{First}, it knows the current rating and has an outage record, from which the current outage probability is known with reasonable accuracy. In the Surabaya data of \cite{akhra2023robust}, for instance, the reported reliabilities of the eleven feeders lie between $0.974$ and $0.983$. \textbf{Second}, it has cheap observable features for every past or simulated day, such as the temperature anomaly, the day-ahead load-forecast anomaly, a storm flag, the day of the week and the season. \textbf{Third}, it has a simulator of the feeder that takes the features of a day, draws its own internal randomness (vehicle arrivals, equipment state and protection behaviour), computes the loading and reports whether the feeder trips. One simulated day is expensive compared with one feature draw, since it involves a power-flow or thermal model and a stochastic arrival process. The planning question is which reserve, through a feeder upgrade or a cap on the charging load, brings the outage probability below the target. This work focuses on the estimation part of this question: how to estimate the outage probability of a feeder at a given reserve with as few simulator runs as possible, together with a reliable and meaningful confidence interval for such a critical task.

\subsection{Notation and Model}
Let $Y=\ind\{\text{outage day}\}$, and let $\theta=\Prob(Y=1)$ be the outage probability of the feeder at the reserve under study. The features of a day form a vector $X\in\cX$ with distribution $F_X$, which can be drawn at cost $c_X$. The simulator returns $Y$ given $X$ at cost $c_Y\gg c_X$. We write $\eta(x)=\Prob(Y=1\mid X=x)$ for the conditional outage probability given the features. We assume a simulator, which can return a continuous severity $S$ and the standardized peak loading of the day, hence $Y=\ind\{S\ge\kappa\}$ for a rating $\kappa$. A higher reserve thus corresponds to a higher $\kappa$. In the analysis that follows, we treats the simulator as a black box that simply returns $Y$.

Planners often associate the outage probability $\theta$ with the following: the expected number of outage days per year (i.e., $365\,\theta$), the probability of at least one outage day in a horizon of $T$ days (i.e., $1-(1-\theta)^T\approx T\theta$), and and the station reliability (i.e., $1-\theta$). All three are linear in $\theta$ to first order, hence a variance reduction obtained for $\theta$ apply to them directly.

\subsection{Monte Carlo and Control Variates}
With $n$ independent simulated days, $\thetahat^{\mc}=\bar Y_n$ is unbiased with $\Var(\thetahat^{\mc})=\theta(1-\theta)/n$. The number of runs required for a relative half-width $\eps$ at confidence $1-\alpha$ is
\begin{equation}
n_{\mc}(\theta,\eps,\alpha)=z_{1-\alpha/2}^2(1-\theta)/(\theta\,\eps^2),
\label{eq:nmc}
\end{equation}
which grows as $1/\theta$. For a $10\%$ relative error at $95\%$ confidence, this is about $400/\theta$ runs, i.e., $4\times10^5$ runs at $\theta=10^{-3}$ and $4\times10^8$ at $10^{-6}$. In other words, the planning has to observe about $400$ outage days in the Monte Carlo simulation to achieve this level of relative error. To that end, a control variate is a random variable $C$ that is simulated together with $Y$, is correlated with $Y$, and has a known mean $\gamma=\E[C]$. For any coefficient $\beta$, the estimator
\begin{equation}
\thetahat^{\cv}_\beta=\bar Y_n-\beta\,(\bar C_n-\gamma)
\label{eq:cv}
\end{equation}
is unbiased. Furthermore, with $\rho=\Corr(Y,C)$ and $\beta^\star=\Cov(Y,C)/\Var(C)$, its variance is \cite{lavenberg1981perspective,glasserman2004monte}
\begin{equation}
\Var(\thetahat^{\cv}_{\beta^\star})=(1-\rho^2)\,\Var(Y)/n .
\label{eq:cvvar}
\end{equation}
Thus, $\rho^2$ is the fraction of variance removed, $1/(1-\rho^2)$ is the factor by which the number of simulator runs can be reduced, and $\sqrt{(1-\rho^2)(1-\theta)/(n\theta)}$ is the relative standard error.

In an earlier study \cite{akhra2023robust}, researcher use the control variate $C=\ind\{P\ge t\}$ for a Gaussian proxy $P$ of the load, with $\gamma=\Phibar(t)=1-\Phi(t)$. In this paper, we use $C=g(X)$ for a learner $g$ trained on a separate set of simulator runs. Its mean $\gamma$ is either known, when $g$ thresholds a feature with a known distribution, or estimated from a pool of unlabelled feature draws, such as the features of every day in the utility's weather archive. Algorithm~\ref{alg:lcv} summarizes the procedure.

\begin{algorithm}[t]
\caption{Learned control variate (LCV)}
\label{alg:lcv}
\KwIn{feature law $F_X$ (cost $c_X$ per draw); simulator $Y\mid X$ (cost $c_Y$); training design $q$; sizes $n_{\mathrm{tr}}$, $n$, $m$; learner class $\mathcal G$}
Draw $n_{\mathrm{tr}}$ training features from $q$, run the simulator on them, and fit $g\in\mathcal G$ with importance weights $F_X/q$ where the fitting criterion requires them (Section~\ref{sec:design})\;
Draw $n$ fresh features from $F_X$, run the simulator, obtain $Y_1,\dots,Y_n$; set $C_l=g(X_l)$\;
\eIf{$\gamma=\E[g(X)]$ is known}{use $\gamma$\;}{draw $m$ further features (no simulator runs) and set $\hat\gamma=\frac{1}{n+m}\big(\sum_{l\le n}g(X_l)+\sum_{l\le m}g(X'_l)\big)$\;}
$\hat\beta\leftarrow\widehat{\Cov}(Y,C)/\widehat{\Var}(C)$ from the $n$ pairs; $\thetahat\leftarrow\bar Y_n-\hat\beta(\bar C_n-\gamma)$ (or $\hat\gamma$)\;
\KwOut{$\thetahat$ and its standard error from \eqref{eq:cvvar} or \eqref{eq:poolvar}.}
\end{algorithm}

\section{Analytical Results}
\label{sec:theory}

Throughout this section, $Y\in\{0,1\}$ with $\theta=\Prob(Y=1)\in(0,1)$, and $X\in\cX$ is the feature vector. 

\subsection{Binary Learners}

\begin{definition}
A binary learner is a measurable $g:\cX\to\{0,1\}$ with recall $\tpr=\Prob(g(X)=1\mid Y=1)$, false-positive rate $\fpr=\Prob(g(X)=1\mid Y=0)$, alarm rate $\gamma=\Prob(g(X)=1)=\theta\tpr+(1-\theta)\fpr$ and precision $\precn=\Prob(Y=1\mid g(X)=1)=\theta\tpr/\gamma$. Following \cite{kearns1994cryptographic,schapire1990strength}, $g$ is a \emph{weak learner} with edge $\eps>0$ if its balanced error rate $\tfrac12((1-\tpr)+\fpr)$ is at most $\tfrac12-\eps$, i.e., $\tpr-\fpr\ge2\eps$. Note that no requirement is placed on accuracy, calibration or the area under the ROC curve.
\end{definition}

With this definition, we have the following:

\begin{theorem}[Variance reduction from a binary learner]
\label{thm:binary}
For a binary learner with $0<\gamma<1$ and $C=g(X)$,
\begin{equation}
\Cov(Y,C)=\theta(1-\theta)(\tpr-\fpr),\quad
\rho^2=\frac{(\tpr-\fpr)^2\,\theta(1-\theta)}{\gamma(1-\gamma)}.
\label{eq:rho2}
\end{equation}
Hence, the CV estimator has a smaller variance than MC if and only if $\tpr\ne\fpr$. A weak learner with edge $\eps$ guarantees $\rho^2\ge4\eps^2\theta(1-\theta)/(\gamma(1-\gamma))>0$.
\end{theorem}
\begin{proof}
$\E[YC]=\Prob(Y=1,g(X)=1)=\theta\tpr$ and $\E[Y]\E[C]=\theta\gamma$, so $\Cov(Y,C)=\theta(\tpr-\gamma)=\theta(1-\theta)(\tpr-\fpr)$. Dividing its square by $\Var(Y)\Var(C)=\theta(1-\theta)\gamma(1-\gamma)$ gives $\rho^2$.
\end{proof}

Since $\rho^2$ depends on $(\tpr-\fpr)^2$, a learner that is worse than chance is as useful as one that is better by the same margin, and only the sign of $\beta^\star$ changes. This means, the only unusable learner is one whose alarms are independent of the outage. More interestingly, we have the following limit:

\begin{corollary}[Rare-event limit]
\label{cor:rare}
For $\tpr>0$, $\rho^2=\tpr\precn\,(1-\fpr/\tpr)^2(1-\theta)/(1-\gamma)$. If $\theta\to0$ while the precision stays bounded away from zero, then $\gamma,\fpr\to0$ and
\begin{equation}
\frac{\Var(\thetahat^{\cv}_{\beta^\star})}{\Var(\thetahat^{\mc})}\;\to\;1-\mathrm{recall}\times\mathrm{precision}.
\label{eq:rarelimit}
\end{equation}
If instead $\fpr$ is held fixed as $\theta\to0$, then $\rho^2=O(\theta)\to0$.
\end{corollary}
\begin{proof}
Substitute $\theta/\gamma=\precn/\tpr$ into \eqref{eq:rho2}. Every factor other than $\tpr\precn$ tends to one when $\gamma,\fpr\to0$. The fixed-$\fpr$ case follows from \eqref{eq:rho2} with $\gamma\to\fpr$.
\end{proof}

Corollary~\ref{cor:rare} gives the operational meaning of a weak learner for rare events. Any weak learner reduces the variance. However, the reduction stays bounded away from zero as the outage becomes rarer only if the false-positive rate shrinks in proportion to $\theta$, i.e., only if the learner has non-vanishing precision. High recall is easy to obtain by predicting an outage every day, whereas precision has to be learned. Hence, a forecaster tuned to accuracy or ROC criteria is not tuned to the quantity that is meaningful for estimation.

\subsection{General Learners and the Ceiling Set by the Features}

With regard to more general learners and features, we have the following:
\begin{theorem}[Factorisation and ceiling]
\label{thm:ceiling}
Let $\eta(x)=\Prob(Y=1\mid X=x)$, and let $g:\cX\to\R$ have $0<\Var(g(X))<\infty$. Then
\begin{equation}
\rho^2(Y,g(X))=\rho^2(\eta(X),g(X))\cdot\rho^2_{\max}
\label{eq:factor}
\end{equation}
where $\rho^2_{\max}= \Var(\eta(X))/\theta(1-\theta)$.
Hence, $\rho^2(Y,g(X))\le\rho^2_{\max}$ for every learner, with equality if and only if $g(X)$ is almost surely an affine function of $\eta(X)$.
\end{theorem}
\begin{proof}
By the tower property, we have $\Cov(Y,g(X))=\Cov(\eta(X),g(X))$. Therefore, $$\rho^2(Y,g)=[\Cov(\eta,g)^2/(\Var(\eta)\Var(g))]\cdot[\Var(\eta)/\Var(Y)],$$ and Cauchy--Schwarz inequality bounds first factor by 1.
\end{proof}

Three consequences follow from this theorem. First, calibration does not matter, since an affine rescaling of a score leaves $\rho^2$ unchanged. Second, a thresholded score is a binary learner and is governed by Theorem~\ref{thm:binary}. Third, the ceiling is a property of the features. The information that a utility collects about each day therefore bounds what any learner, weak or strong, could give in terms of the estimation efficiency.

For Gaussian proxy, we have the following result:

\begin{proposition}[A Gaussian proxy loses its effect for rare events]
\label{prop:gauss}
Let $(S,P)$ be bivariate normal with standard margins and correlation $r\in(-1,1)$, $Y=\ind\{S\ge\kappa\}$ and $C_t=\ind\{P\ge t\}$. Then $\sup_{t}\rho^2(Y,C_t)\to0$ as $\kappa\to\infty$.
\end{proposition}
\begin{proof}[Proof sketch]
For a binary control with nonnegative covariance, \eqref{eq:rho2} gives $\rho^2\le\tpr\precn/((1-\theta)(1-\gamma))$. Take $r\ge0$, since the case $r<0$ follows by replacing $P$ with $-P$, and let $c=(1+r)/2<1$. For $t\ge c\kappa$, the recall $\Prob(P\ge t\mid S\ge\kappa)$ vanishes as $\kappa\to\infty$. Indeed, conditionally on $S=s\approx\kappa$, the proxy is $N(rs,1-r^2)$, and $t-r\kappa\ge\kappa(1-r)/2$. For $0\le t<c\kappa$, the precision is at most $\Phibar(\kappa)/\Phibar(c\kappa)$, which tends to zero by the Mills ratio. For $t<0$, we use the complementary control $1-C_t$, whose recall vanishes in the same way. The three bounds are uniform in $t$.
\end{proof}

The mechanism behind Proposition~\ref{prop:gauss} is the asymptotic tail independence of the bivariate normal distribution \cite{sibuya1960bivariate}. The proxy and the load agree on moderate days but disagree on the extreme days that determine the outage. By Corollary~\ref{cor:rare}, a learner whose precision does not vanish keeps a constant fraction of the variance. Whether such a learner exists depends on the mechanism of the outage. If extreme loads are driven by observable extremes (heat waves, forecast misses, storms), the features stay informative in the tail. If they are driven by unobserved noise, no estimator using those features is helpful.

\subsection{Estimated Coefficient, Unlabelled Pool, and Budget}

To make our results practical, we extend our analysis to the estimated coefficient regime and report the following:

\begin{theorem}[Estimated coefficient and unlabelled pool]
\label{thm:pool}
Let the $n$ labelled pairs and the $m$ unlabelled feature draws be independent, with $g$ fixed, $\Var(C)>0$ and $\rho=\Corr(Y,g(X))$.
\begin{enumerate}
\item[(a)] With $\hat\beta$ estimated from the same $n$ pairs and $\gamma$ known, $\E[\thetahat]=\theta+O(1/n)$ and $\sqrt n(\thetahat-\theta)\to N(0,(1-\rho^2)\Var(Y))$.
\item[(b)] With $\hat\gamma$ as in Algorithm~\ref{alg:lcv}, $\thetahat_\beta$ is unbiased for every fixed $\beta$, and
\begin{equation}
\Var(\thetahat_{\beta^\star})=\frac{\Var(Y)}{n}\Big(1-\rho^2\frac{m}{n+m}\Big).
\label{eq:poolvar}
\end{equation}
\item[(c)] With budget $B=n(c_Y+c_X)+mc_X$, \eqref{eq:poolvar} is minimised at $(n+m)/n=\sqrt{c_Y\rho^2/(c_X(1-\rho^2))}$ whenever this value exceeds 1. The variance ratio to MC at the same budget is $\Var^\star/\Var^{\mc}=\big(\sqrt{(1-\rho^2)c_Y}+\sqrt{\rho^2c_X}\big)^2/(c_Y+c_X)$, which tends to $1-\rho^2$ as $c_X/c_Y\to0$.
\end{enumerate}
\end{theorem}
\begin{proof}[Proof sketch]
(a) is the classical result for control variates with estimated coefficients \cite{lavenberg1981perspective,nelson1990control}. (b) Write $\bar C_n-\hat\gamma=m/(n+m) \cdot (\bar C_n-\bar C'_m)$, where the pool mean $\bar C'_m$ is independent of the pairs. Then expand the variance and substitute $\beta^\star$. (c) With $s=(n+m)/n$, the variance equals $\Var(Y)f(s)/B$ with $f(s)=(c_Y+sc_X)(1-\rho^2+\rho^2/s)$. This function is convex and is minimised at the stated value.
\end{proof}

Part (b) coincides with the two-fidelity estimator \cite{peherstorfer2016optimal} and with prediction-powered inference with a tuned power parameter \cite{angelopoulos2023ppi}. Combined with Theorems~\ref{thm:binary} and~\ref{thm:ceiling}, it gives the variance in closed form for a rare binary outcome. In practice, the pool is the utility's weather archive, on which the control is evaluated at no simulation cost. Part (c) then gives the number of archive days that is worth one simulator run.

\subsection{Training Design}
\label{sec:design}

Finally, we address the training nature of the learned CVs. In the rare regime, a learner trained on features drawn from the nominal distribution $F_X$ sees almost no outage and almost no extreme day. The following proposition shows that the training features may instead be drawn from any design.

\begin{proposition}[Validity under any training design]
\label{prop:design}
Let $g$ be any function of a training sample whose features are drawn from a distribution $q$ on $\cX$ and labelled by the simulator, independently of the $n$ estimation pairs drawn from $F_X$. Then, conditionally on $g$, Theorem~\ref{thm:pool} applies to the estimator of Algorithm~\ref{alg:lcv}, with $\rho^2$ computed under $F_X$. Furthermore, if $F_X\ll q$ with weights $w=dF_X/dq$, the weighted training-sample ratios $\hat\tpr=\sum_lw_lY_lg(X_l)/\sum_lw_lY_l$ and $\hat\precn=\sum_lw_lY_lg(X_l)/\sum_lw_lg(X_l)$ are consistent for the nominal recall and precision. Hence, $\hat\tpr\hat\precn$ is a consistent estimate of the limit \eqref{eq:rarelimit}.
\end{proposition}
\begin{proof}
Conditionally on $g$, the estimation pairs are i.i.d.\ from the nominal distribution and $g$ is fixed, so the classical results apply. The weighted ratios are importance-sampling estimates \cite{kahn1953methods,shimodaira2000improving}.
\end{proof}

In practice, the $n_{\mathrm{tr}}$ expensive training runs are placed where the outages are, by drawing the training features from a shifted design that over-represents hot days and high forecasts. Since the simulator returns the severity $S$, the label $\ind\{S\ge\kappa\}$ can be formed afterwards for any rating $\kappa$. One training set therefore serves every rarity level. Score learners are fitted to $S$ by regression, while the threshold of a binary learner is chosen to maximise $\hat\tpr\hat\precn$. For a utility, the simulation planning used to train the control is thus a one-time investment, which can be reused when the ratings, the reliability targets or the stations under study change, which further emphasize the practical value of our proposed framework.

\section{Numerical Experiments}
\label{sec:experiments}

\subsection{Simulator and Experiment Setup}

Each simulated day has a temperature anomaly $x_1\sim N(0,1)$, a day-ahead load-forecast anomaly $x_2\sim N(0,1)$ and a storm flag $x_3\sim\mathrm{Bernoulli}(0.15)$. The simulator returns the standardised peak feeder loading
\begin{align}
S &=a_1x_1+a_2\max(x_1,0)^2+a_3x_2\nonumber \\& +a_4x_3\max(\eps_2,0) +a_0\eps_1,
\label{eq:sim}
\end{align}
with internal noise $\eps_1,\eps_2\sim N(0,1)$, and $Y=\ind\{S\ge\kappa\}$. The coefficients $(a_1,a_2,a_3,a_4,a_0)=(0.45,0.20,0.55,0.40,0.30)$ represent a nonlinear heat effect, a linear forecast effect, a storm effect that only sometimes materialises, and unobserved noise. The forecast anomaly thus acts as a Gaussian proxy of the load, with latent correlation $0.61$. The rating $\kappa$ is calibrated by quadrature so that $\theta\in\{10^{-3},10^{-4},10^{-5},10^{-6}\}$ exactly. Since the noise structure is known, $\eta(x)$ is also available by quadrature, and every $\rho^2$ is evaluated exactly by importance sampling over the features. It is worth noting that our estimators never use this information.

We compare plain MC with five controls: (1) the Gaussian proxy (GP) $\ind\{x_2\ge t\}$, with $t$ chosen to maximise its exact $\rho^2$; (2) a stump (ST) $\ind\{x_1\ge t\}$, i.e., a one-feature weak learner with $t$ chosen by the weighted $\hat\tpr\hat\precn$ criterion; (3) a deliberately misspecified linear learner (LL), i.e., least squares of $S$ on $X$, with score $\Phibar((\kappa-\hat s(x))/\hat\sigma)$; (4) gradient-boosted trees (BT) regressing $S$ on $X$ \cite{pedregosa2011scikit}, a standard strong learner for tabular data, with the same score; and (5) the Bayes posterior (BP) $\eta(X)$, which attains the ceiling. The learners are trained once on $n_{\mathrm{tr}}=10^4$ runs from the tail-focused design $x_1\sim N(2,1.5^2)$, $x_2\sim N(1,1.5^2)$, and are reused at every $\theta$. For comparison, they are also trained on $10^4$ nominal runs. The control means are computed once by importance sampling (relative standard error below $0.3\%$).

\subsection{Estimation Results}

\begin{figure*}[t]
\centering
\includegraphics[width=\textwidth]{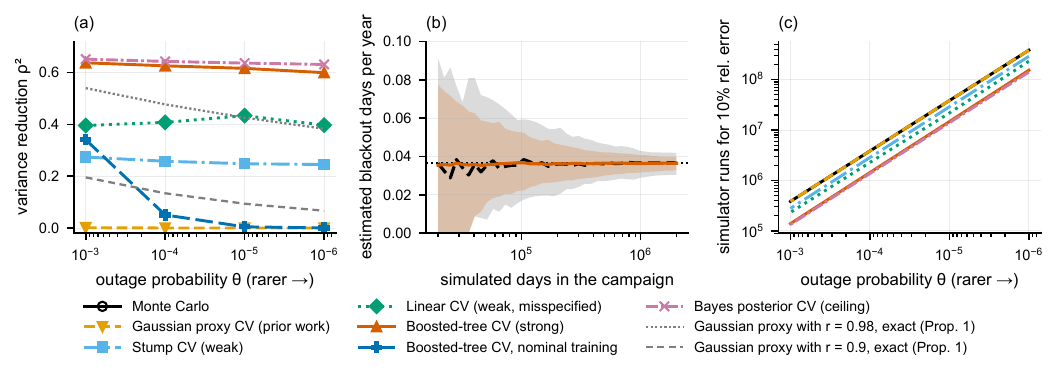}
\caption{(a) Variance reduction $\rho^2$ of each control as the outage becomes rarer. (b) Running estimate at $\theta=10^{-4}$ as the sample grows: median (line) and 95\% range (shaded) over 400 samples, for Monte Carlo (black, dashed) and the boosted-tree control on the same days (solid); The true value is shown as the dotted line. (c) Simulator runs needed to achieve a 10\% relative error.}
\label{fig:ablation}
\end{figure*}

\begin{table}[t]
\centering
\caption{Replication study at about 20 outages per replication}
\label{tab:ablation}
\scriptsize
\setlength{\tabcolsep}{4pt}
\begin{tabular*}{\columnwidth}{@{\extracolsep{\fill}}llcccccc@{}}
\toprule
$\theta$ & Estimator & $\rho^2$ & $\frac{1}{1-\rho^2}$ & RMSE & theory & bias & cov.\\
\midrule
$10^{-3}$ & MC & 0.000 & 1.00 & 0.233 & 0.223 & +0.015 & 0.94\\
 & GP & 0.001 & 1.00 & 0.233 & 0.223 & +0.015 & 0.94\\
 & ST (weak) & 0.274 & 1.38 & 0.191 & 0.190 & +0.008 & 0.95\\
 & LL (weak) & 0.396 & 1.66 & 0.175 & 0.174 & +0.005 & 0.94\\
 & BT & 0.638 & 2.76 & 0.135 & 0.135 & +0.008 & 0.95\\
 & BP & 0.652 & 2.87 & 0.134 & 0.132 & +0.010 & 0.95\\
\midrule
$10^{-4}$ & MC & 0.000 & 1.00 & 0.214 & 0.224 & +0.008 & 0.96\\
 & GP & 0.000 & 1.00 & 0.214 & 0.224 & +0.008 & 0.96\\
 & ST (weak) & 0.258 & 1.35 & 0.189 & 0.193 & -0.001 & 0.94\\
 & LL (weak) & 0.408 & 1.69 & 0.181 & 0.172 & -0.004 & 0.92\\
 & BT & 0.626 & 2.67 & 0.141 & 0.137 & -0.012 & 0.94\\
 & BP & 0.643 & 2.80 & 0.136 & 0.134 & -0.010 & 0.93\\
\midrule
$10^{-5}$ & MC & 0.000 & 1.00 & 0.240 & 0.224 & +0.037 & 0.93\\
 & GP & 0.000 & 1.00 & 0.240 & 0.224 & +0.037 & 0.93\\
 & ST (weak) & 0.248 & 1.33 & 0.200 & 0.194 & +0.036 & 0.95\\
 & LL (weak) & 0.434 & 1.77 & 0.165 & 0.168 & +0.007 & 0.94\\
 & BT & 0.616 & 2.61 & 0.125 & 0.139 & +0.009 & 0.97\\
 & BP & 0.637 & 2.75 & 0.128 & 0.135 & +0.012 & 0.96\\
\midrule
$10^{-6}$ & MC & 0.000 & 1.00 & 0.235 & 0.224 & +0.083 & 0.98\\
 & GP & 0.000 & 1.00 & 0.235 & 0.224 & +0.082 & 0.98\\
 & ST (weak) & 0.245 & 1.32 & 0.176 & 0.194 & +0.068 & 0.97\\
 & LL (weak) & 0.397 & 1.66 & 0.178 & 0.174 & +0.037 & 0.95\\
 & BT & 0.600 & 2.50 & 0.131 & 0.141 & +0.025 & 0.95\\
 & BP & 0.631 & 2.71 & 0.126 & 0.136 & +0.017 & 0.95\\
\bottomrule
\end{tabular*}

\end{table}

Fig.~\ref{fig:ablation}(a) and Table~\ref{tab:ablation} present the main result. The Gaussian proxy removes less than $0.2\%$ of the variance at $\theta=10^{-3}$ and nothing beyond. Even a hypothetical proxy with latent correlation $0.98$ decays from $0.54$ to $0.38$ over the range (grey curves), as predicted by Proposition~\ref{prop:gauss}. The learned controls, in contrast, keep their effect at every level. The stump keeps $\rho^2\approx0.25$ ($1.3$ times fewer runs), and the misspecified linear learner keeps about $0.40$ ($1.7$ times). The boosted trees keep $0.60$ to $0.64$ ($2.5$ to $2.8$ times), which is within $5\%$ of the ceiling of $0.63$ to $0.65$ set by the features. Table~\ref{tab:ablation} repeats each estimator over 400, 300, 120 and 60 replications at the four rarity levels, with $n=20/\theta$ simulated days per replication, i.e., about 20 outages. For each estimator, it reports the exact $\rho^2$, the equivalent-runs multiplier $1/(1-\rho^2)$, the empirical relative RMSE and its theoretical value $\sqrt{(1-\rho^2)(1-\theta)/(n\theta)}$ from \eqref{eq:cvvar}, the relative bias, and the coverage of the plug-in $95\%$ interval. In Table~\ref{tab:ablation}, every empirical relative RMSE is within two standard errors of its theoretical value, e.g., $0.21$ to $0.24$ for MC against $0.224$. The boosted-tree control reduces it to $0.13$ to $0.14$, while the $95\%$ intervals keep a coverage between $0.92$ and $0.98$. At $\theta=10^{-6}$, a $10\%$ relative error requires $3.8\times10^8$ runs with MC and $1.5\times10^8$ with the boosted-tree control, for a one-time training cost of $1.2\times10^4$ runs.

Fig.~\ref{fig:ablation}(b) translates this result into more practical insights. At $\theta=10^{-4}$, i.e., $0.037$ expected blackout days per year, the $95\%$ range of the MC estimate over 400 plannings is still $\pm0.021$ days per year after $10^5$ simulated days and $\pm0.007$ after $10^6$. The control variate reaches the same ranges after about $5\times10^4$ and $4\times10^5$ days, and neither estimator shows a bias. Very short replications are the exception. With only about two observed outages, $\hat\beta$ is fitted to a handful of positive days, so the control may increase the variance and its plug-in interval undercovers. From about twenty outages onward, the coverage is close to nominal.

Furthermore, it appears That the training design plays crucial role as much as the learner. The same boosted learner trained on $10^4$ nominal runs has $\rho^2=0.34$ at $\theta=10^{-3}$, $0.05$ at $10^{-4}$ and zero below, since its training set contains no extreme day and tree ensembles extrapolate as a constant. The tail-focused design, in contrast, leaves hundreds of outage days in the training set even at $\theta=10^{-6}$. Finally, with the control mean taken from a pool of $m=10n$ unlabelled draws, the loss relative to a known mean is about a tenth of $\rho^2$, as \eqref{eq:poolvar} predicts.

\subsection{Case Study: Blackout Map of Surabaya}
\label{sec:case}

We now apply the estimator to a planning decision on the Surabaya network \cite{akhra2023robust}, which has 11 candidate feeders and 31 sub-districts. For each feeder $j$, the earlier work reports a masked daily peak load with mean $\mu_j$ and standard deviation $\sigma_j$, a rating $z_j$, and a reliability $p_j$ between $0.974$ and $0.983$ at the current headroom $h_{0j}=(z_j-\mu_j)/\sigma_j\approx2$. We place a charging station on every feeder. The planner then chooses the headroom $h_j$, through a feeder upgrade or a cap on the charging load, so that the daily overload probability $\theta_j(h_j)$ meets the one-day-in-ten-years target $\theta^\star=2.7\times10^{-4}$. The daily peak load is modelled as $\mu_j+\sigma_j\tilde S_j$, where $\tilde S_j$ is the severity \eqref{eq:sim} standardised so that $\Prob(\tilde S_j\ge h_{0j})=1-p_j$. The simulator therefore reproduces the reported reliability of every feeder and differs from the Gaussian model only in the tail. 

We compare three design procedures: (a) the Gaussian load model, which requires $h=\Phibar^{-1}(\theta^\star)=3.46$ for every feeder; (b) imprecise estimates, i.e., plain MC with $n=10^4$ simulated days per feeder (the largest sample of the earlier work) and the reserve sized to the point estimate; and (c) precise estimates, i.e., the boosted-tree control variate on the same days, with the reserve sized to the smallest headroom whose $95\%$ upper confidence limit meets the target. Since the simulator returns the severity, one replication gives the whole curve $\theta_j(h)$.

\begin{figure*}[t]
\centering
\includegraphics[width=\textwidth]{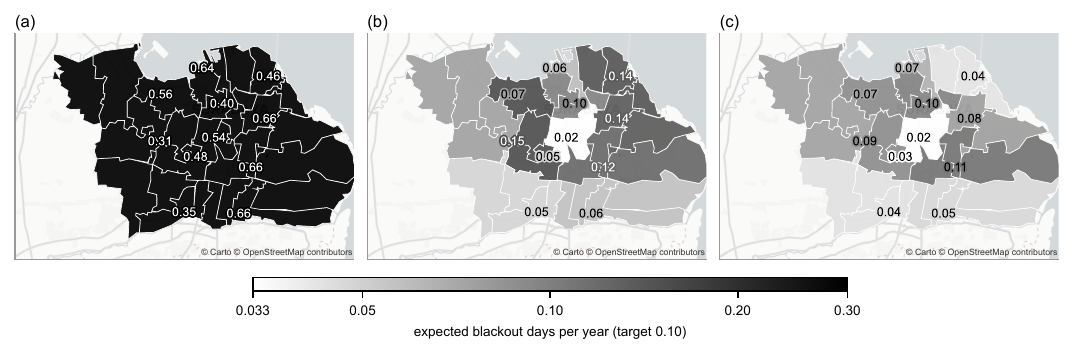}
\caption{Expected blackout days per year after a charging station is added and the reserve is sized for one day in ten years by (a) the Gaussian load model of the earlier work, (b) Monte Carlo with $10^4$ days and the point estimate, and (c) the learned control variate on the same days with a 95\% upper limit.}
\label{fig:map}
\end{figure*}

Fig.~\ref{fig:map} shows the true expected blackout days per year of each design. Table~\ref{tab:map} summarises 200 replications per feeder by the city total of feeder-blackout days per year, the share of feeders whose design ends above the target and above twice the target, the 90th percentile of the blackout days per year of a feeder (averaged over the feeders), and the share of the population served by a feeder above the target. Under the Gaussian model, every feeder is under-designed. The required headroom is $4.05$ to $4.46$, and the true blackout rate is $0.31$ to $0.66$ days per year, i.e., three to seven times the target, for a city total of $5.7$ feeder-blackout days per year. This is the design consequence of Proposition~\ref{prop:gauss}: the Gaussian model agrees with the simulator at the current operating point but underestimates the reserve needed in the tail. With imprecise estimates (panel (b)), $46\%$ of the feeders end above the target and $8\%$ above twice the target. Moreover, the planner cannot tell which feeders these are. In the example shown, for instance, some districts (Darmo Permai, Kenjeran and Rayon Ploso) are at $0.14$ to $0.15$ days per year, while other district (Embong Wungu) is over-built at $0.02$. With precise estimates (panel (c)), $15\%$ of the feeders remain above the target. The city total falls from $1.17$ to $0.77$ feeder-blackout days per year, while the share of the population on a feeder above the target falls from $46\%$ to $15\%$. Table~\ref{tab:map} also breaks down the contribution the two components. The confidence limit protects the design, since MC with the same limit reaches $0.78$ city days. The control variate then reduces the remaining failures, from $24\%$ to $15\%$ of the feeders above the target, and the 90th percentile from $0.14$ to $0.11$ days per year.

\begin{table}[t]
\centering
\caption{City-wide outcome over 200 replications of $10{,}000$ days}
\label{tab:map}
\scriptsize
\setlength{\tabcolsep}{4pt}
\begin{tabular*}{\columnwidth}{@{\extracolsep{\fill}}lccccc@{}}
\toprule
Procedure & city days/yr & $>$target & $>2\times$ & 90th pct. & pop.\\
\midrule
Gaussian model & 5.73 & 1.00 & 1.00 & 0.521 & 1.00\\
MC, point est. & 1.17 & 0.46 & 0.08 & 0.192 & 0.46\\
CV, point est. & 1.16 & 0.49 & 0.02 & 0.162 & 0.49\\
MC, 95\% limit & 0.78 & 0.24 & 0.02 & 0.139 & 0.23\\
CV, 95\% limit & 0.77 & 0.15 & 0.00 & 0.110 & 0.15\\
\bottomrule
\end{tabular*}

\end{table}
The cost side of the planners is illustrated on the largest feeder (Embong Wungu, with $\mu=28{,}979$, $\sigma=3{,}622$, rating $36{,}224$). The exact requirement is an upgrade of $8{,}415$ units ($29\%$ of the base load), whereas the Gaussian model asks for $5{,}274$ ($18\%$). Under the confidence-limit rule, MC over-designs the feeder by $1{,}420$, $675$ and $330$ units on average for samples of $2\times10^4$, $10^5$ and $5\times10^5$ days, against $816$, $459$ and $240$ units for the control variate. The control-variate sample thus matches the decision quality of an MC sample about two to three times larger, close to the factor $1/(1-\rho^2)\approx2.7$ from \eqref{eq:cvvar}. Both estimators keep the true outage probability below the target in at least $93\%$ of the samples.

\section{Discussion}
\label{sec:discussion}

\subsection{Implications for Grid Operators and City Planners}

The results lead to a handful of implications for a distribution utility or a city that plans its EV infrastructure.

First, reliability targets could be certified with far fewer simulations. Since the planning quantities are linear in $\theta$ (Section~\ref{sec:setting}), the variance reduction transfers to them unchanged. In the Surabaya case, the control variate reaches the decision quality of an MC sample two to three times larger. Combined with an upper confidence limit, it also reduces the city-wide blackout days by a third, and the population on an under-designed feeder from $46\%$ to $15\%$, compared with MC point estimates. Furthermore, the load model plays important role as much as the estimator. The Gaussian model of the earlier work reproduces the current outage rate of some districts (including Embong Wungu), yet it underestimates the required upgrade by about $3{,}100$ units. The resulting risk is more than five times the accepted level.

Next, investment in features should come before investment in models. By Theorem~\ref{thm:ceiling}, the information collected about each day bounds what any learner could achieve. In our simulator, halving or doubling the unobserved noise moved the ceiling from $0.64$ to $0.79$ or $0.38$ for every learner. Hence, feeder loading histories, maintenance schedules, distributed-generation output and the load forecasts that utilities already produce \cite{hu2025evload,chen2025reliability} could raise the ceiling, while a more elaborate model on the same features could not.

Furthermore, the training budget should be spent where the outages are. A training sample placed on hot, high-forecast days keeps its value at every reliability target and can be reused when the ratings or the targets change. A sample placed on typical days, in contrast, is wasted in the rare regime. Finally, the control depends only on the features of a day, so it can be evaluated on a decade of weather records at no simulation cost. Theorem~\ref{thm:pool}(c) then gives the number of such days that is worth one simulator run. Together, these points suggest a simple protocol for an operator: train once on a tail-focused sample, evaluate the control on the archive, size the reserve with an upper confidence limit, and commit only after the sample contains at least about twenty outages.

\subsection{Limitations and Future Work}

We also highlight a few limitations of this work and future work directions. First, we are limited by the use of a synthetic simulator and synthetic features, which are chosen so that the exact posterior is available for evaluation. In the Surabaya case, the simulator matches the reported reliability of every feeder, but the shape of its tail is assumed. A feeder model driven by real weather and outage records is therefore the natural next step. Furthermore, outages at different feeders are modelled independently, although a city-wide storm correlates them, which we ignore in this study. Finally, the tail-focused design was also chosen by hand. An adaptive design could alleviate this choice, which might allow us to combine our current learned control variate results with  learned importance-sampling proposal \cite{arief2021deep}, which we plan to explore in our future work.

\section{Conclusion}
\label{sec:conclusion}

In this paper, we have shown that a classifier of grid outages trained on cheap observable features yields a valid control variate, which reduces the variance for every weak learner. Its effect is governed by recall and precision through an exact formula and is bounded by a ceiling that the features alone determine. Unlike a Gaussian proxy of the load, it does not vanish as the outage becomes rarer, provided that the learner is trained where the outages are. On a feeder-overload simulator, the learned controls removed between a quarter and nearly two thirds of the Monte Carlo variance at outage probabilities from $10^{-3}$ to $10^{-6}$. On the Surabaya network, the Gaussian load model yields every feeder three to seven times the one-day-in-ten-years target, while point-estimate designs from plain Monte Carlo makes almost half of the feeders achieve above the target. At the same simulation budget, the learned control variate with a confidence limit also reduces the city-wide blackout days by a third. Compared with Monte Carlo under the same rule, it also frees up the over-design of the reserve by up to $40\%$. The proposed method thus makes the certification of rare overload probabilities affordable for a utility, at the cost of a one-time training sample and a pool of historical feature data, which opens huge potential to design a more reliable grid against rare-event outages.


\IfFileExists{IEEEtran.bst}{\bibliographystyle{IEEEtran}}{\bibliographystyle{ieeetr}}
\bibliography{refs}

@misc{akhra2023robust,
  title         = {A Robust Electric Vehicle Charging Station Optimization under Rare Event Disruptions},
  author        = {Akhra, Yan and Vanany, Iwan and Arief, Mansur},
  year          = {2026},
  howpublished  = {arXiv:2307.05470v},
  doi           = {10.48550/arXiv.2307.05470}
}

@article{lacommare2006cost,
  title   = {Cost of power interruptions to electricity consumers in the {United States} ({US})},
  author  = {LaCommare, Kristina Hamachi and Eto, Joseph H.},
  journal = {Energy},
  volume  = {31},
  number  = {12},
  pages   = {1845--1855},
  year    = {2006},
  doi     = {10.1016/j.energy.2006.02.008}
}

@techreport{taskforce2004blackout,
  title       = {Final Report on the {August} 14, 2003 Blackout in the {United States} and {Canada}: Causes and Recommendations},
  author      = {{U.S.--Canada Power System Outage Task Force}},
  institution = {U.S. Department of Energy and Natural Resources Canada},
  year        = {2004},
  month       = apr
}

@article{muratori2018impact,
  title   = {Impact of uncoordinated plug-in electric vehicle charging on residential power demand},
  author  = {Muratori, Matteo},
  journal = {Nature Energy},
  volume  = {3},
  number  = {3},
  pages   = {193--201},
  year    = {2018},
  doi     = {10.1038/s41560-017-0074-z}
}

@article{clementnyns2010impact,
  title   = {The impact of charging plug-in hybrid electric vehicles on a residential distribution grid},
  author  = {Clement-Nyns, Kristien and Haesen, Edwin and Driesen, Johan},
  journal = {IEEE Transactions on Power Systems},
  volume  = {25},
  number  = {1},
  pages   = {371--380},
  year    = {2010},
  doi     = {10.1109/TPWRS.2009.2036481}
}

@article{powell2022charging,
  title   = {Charging infrastructure access and operation to reduce the grid impacts of deep electric vehicle adoption},
  author  = {Powell, Siobhan and Cezar, Gustavo Vianna and Min, Liang and Azevedo, In{\^e}s M. L. and Rajagopal, Ram},
  journal = {Nature Energy},
  volume  = {7},
  number  = {10},
  pages   = {932--945},
  year    = {2022},
  doi     = {10.1038/s41560-022-01105-7}
}

@article{dobson2007complex,
  title   = {Complex systems analysis of series of blackouts: Cascading failure, critical points, and self-organization},
  author  = {Dobson, Ian and Carreras, Benjamin A. and Lynch, Vickie E. and Newman, David E.},
  journal = {Chaos},
  volume  = {17},
  number  = {2},
  pages   = {026103},
  year    = {2007},
  doi     = {10.1063/1.2737822}
}

@article{hines2009large,
  title   = {Large blackouts in {North America}: Historical trends and policy implications},
  author  = {Hines, Paul and Apt, Jay and Talukdar, Sarosh},
  journal = {Energy Policy},
  volume  = {37},
  number  = {12},
  pages   = {5249--5259},
  year    = {2009},
  doi     = {10.1016/j.enpol.2009.07.049}
}

@book{billinton1996reliability,
  title     = {Reliability Evaluation of Power Systems},
  author    = {Billinton, Roy and Allan, Ronald N.},
  edition   = {2nd},
  publisher = {Plenum Press},
  address   = {New York},
  year      = {1996}
}

@article{kahn1953methods,
  title   = {Methods of reducing sample size in {Monte Carlo} computations},
  author  = {Kahn, Herman and Marshall, Andy W.},
  journal = {Journal of the Operations Research Society of America},
  volume  = {1},
  number  = {5},
  pages   = {263--278},
  year    = {1953}
}

@article{shimodaira2000improving,
  title   = {Improving predictive inference under covariate shift by weighting the log-likelihood function},
  author  = {Shimodaira, Hidetoshi},
  journal = {Journal of Statistical Planning and Inference},
  volume  = {90},
  number  = {2},
  pages   = {227--244},
  year    = {2000}
}

@article{heidelberger1995fast,
  title   = {Fast simulation of rare events in queueing and reliability models},
  author  = {Heidelberger, Philip},
  journal = {ACM Transactions on Modeling and Computer Simulation},
  volume  = {5},
  number  = {1},
  pages   = {43--85},
  year    = {1995}
}

@article{pedregosa2011scikit,
  title   = {Scikit-learn: Machine learning in {Python}},
  author  = {Pedregosa, Fabian and Varoquaux, Ga{\"e}l and Gramfort, Alexandre and Michel, Vincent and Thirion, Bertrand and Grisel, Olivier and Blondel, Mathieu and Prettenhofer, Peter and Weiss, Ron and Dubourg, Vincent and others},
  journal = {Journal of Machine Learning Research},
  volume  = {12},
  pages   = {2825--2830},
  year    = {2011}
}

@book{asmussen2007stochastic,
  title={Stochastic Simulation: Algorithms and Analysis},
  author={Asmussen, S{\o}ren and Glynn, Peter W},
  year={2007},
  publisher={Springer}
}

@book{glasserman2004monte,
  title={Monte Carlo Methods in Financial Engineering},
  author={Glasserman, Paul},
  year={2004},
  publisher={Springer}
}

@article{lavenberg1981perspective,
  title={A perspective on the use of control variables to increase the efficiency of {M}onte {C}arlo simulations},
  author={Lavenberg, Stephen S and Welch, Peter D},
  journal={Management Science},
  volume={27},
  number={3},
  pages={322--335},
  year={1981},
  publisher={INFORMS}
}

@article{nelson1990control,
  title={Control variate remedies},
  author={Nelson, Barry L},
  journal={Operations Research},
  volume={38},
  number={6},
  pages={974--992},
  year={1990},
  publisher={INFORMS}
}

@article{hesterberg1998control,
  title={Control variates for probability and quantile estimation},
  author={Hesterberg, Timothy C and Nelson, Barry L},
  journal={Management Science},
  volume={44},
  number={9},
  pages={1295--1312},
  year={1998},
  publisher={INFORMS}
}

@article{oates2017control,
  title={Control functionals for {M}onte {C}arlo integration},
  author={Oates, Chris J and Girolami, Mark and Chopin, Nicolas},
  journal={Journal of the Royal Statistical Society: Series B},
  volume={79},
  number={3},
  pages={695--718},
  year={2017}
}

@article{peherstorfer2016optimal,
  title={Optimal model management for multifidelity {M}onte {C}arlo estimation},
  author={Peherstorfer, Benjamin and Willcox, Karen and Gunzburger, Max},
  journal={SIAM Journal on Scientific Computing},
  volume={38},
  number={5},
  pages={A3163--A3194},
  year={2016}
}

@article{peherstorfer2018survey,
  title={Survey of multifidelity methods in uncertainty propagation, inference, and optimization},
  author={Peherstorfer, Benjamin and Willcox, Karen and Gunzburger, Max},
  journal={SIAM Review},
  volume={60},
  number={3},
  pages={550--591},
  year={2018}
}

@article{angelopoulos2023prediction,
  title={Prediction-powered inference},
  author={Angelopoulos, Anastasios N and Bates, Stephen and Fannjiang, Clara and Jordan, Michael I and Zrnic, Tijana},
  journal={Science},
  volume={382},
  number={6671},
  pages={669--674},
  year={2023}
}

@article{angelopoulos2023ppi,
  title={{PPI++}: Efficient prediction-powered inference},
  author={Angelopoulos, Anastasios N and Duchi, John C and Zrnic, Tijana},
  journal={arXiv preprint arXiv:2311.01453},
  year={2023}
}

@article{schapire1990strength,
  title={The strength of weak learnability},
  author={Schapire, Robert E},
  journal={Machine Learning},
  volume={5},
  number={2},
  pages={197--227},
  year={1990}
}

@article{kearns1994cryptographic,
  title={Cryptographic limitations on learning {B}oolean formulae and finite automata},
  author={Kearns, Michael and Valiant, Leslie},
  journal={Journal of the ACM},
  volume={41},
  number={1},
  pages={67--95},
  year={1994}
}

@inproceedings{arief2021deep,
  title={Deep probabilistic accelerated evaluation: A robust certifiable rare-event simulation methodology for black-box safety-critical systems},
  author={Arief, Mansur and Bai, Yuanlu and Ding, Wenhao and He, Shengyi and Huang, Zhiyuan and Lam, Henry and Zhao, Ding},
  booktitle={International Conference on Artificial Intelligence and Statistics (AISTATS)},
  pages={595--603},
  year={2021},
  organization={PMLR}
}

@article{sibuya1960bivariate,
  title={Bivariate extreme statistics, {I}},
  author={Sibuya, Masaaki},
  journal={Annals of the Institute of Statistical Mathematics},
  volume={11},
  number={2},
  pages={195--210},
  year={1960}
}

@article{guikema2014predicting,
  title={Predicting hurricane power outages to support storm response planning},
  author={Guikema, Seth D and Nateghi, Roshanak and Quiring, Steven M and Staid, Andrea and Reilly, Allison C and Gao, Michael},
  journal={IEEE Access},
  volume={2},
  pages={1364--1373},
  year={2014}
}

@article{nateghi2014power,
  title={Power outage estimation for tropical cyclones: Improved accuracy with simpler models},
  author={Nateghi, Roshanak and Guikema, Seth D and Quiring, Steven M},
  journal={Risk Analysis},
  volume={34},
  number={6},
  pages={1069--1078},
  year={2014}
}

@article{lukuyu2024managing,
  title={Managing grid impacts from increased electric vehicle adoption in African cities},
  author={Lukuyu, June and Shirley, Rebekah and Taneja, Jay},
  journal={Scientific Reports},
  volume={14},
  number={1},
  pages={24320},
  year={2024},
  publisher={Nature Publishing Group UK London}
}

@article{nutkani2024impact,
  title={Impact of EV charging on electrical distribution network and mitigating solutions--A review},
  author={Nutkani, Inam and Toole, Hamish and Fernando, Nuwantha and Andrew, Loh Poh Chiang},
  journal={IET Smart Grid},
  volume={7},
  number={5},
  pages={485--502},
  year={2024},
  publisher={Wiley Online Library}
}

@article{khan2024impact,
  title={Impact assessment of electric vehicle charging on distribution networks},
  author={Khan, Mohd Owais and Kirmani, Sheeraz and Rihan, Mohd},
  journal={Renewable Energy Focus},
  volume={50},
  pages={100599},
  year={2024},
  publisher={Elsevier}
}

@article{south2023regularized,
  title={Regularized zero-variance control variates},
  author={South, Leah F and Oates, Chris J and Mira, Antonietta and Drovandi, Christopher},
  journal={Bayesian Analysis},
  volume={18},
  number={3},
  pages={865--888},
  year={2023},
  publisher={International Society for Bayesian Analysis}
}

@article{unterluggauer2022electric,
  title={Electric vehicle charging infrastructure planning for integrated transportation and power distribution networks: A review},
  author={Unterluggauer, Tim and Rich, Jeppe and Andersen, Peter Bach and Hashemi, Seyedmostafa},
  journal={ETransportation},
  pages={100163},
  year={2022},
  publisher={Elsevier}
}

@inproceedings{sun2025hosting,
  title     = {Hosting capacity assessment of distribution networks with {PV} and {EV} charging stations},
  author    = {Sun, Jiazheng and Hu, Xiaoxiao and Liu, Hengyu and Zhang, Zhi and Peng, Hao and Yang, Zhihao},
  booktitle = {2025 9th International Conference on Smart Grid and Smart Cities (ICSGSC)},
  pages     = {155--159},
  year      = {2025},
  organization = {IEEE}
}

@inproceedings{hu2025evload,
  title     = {Electric vehicle load forecasting based on probabilistic modeling and deep learning},
  author    = {Hu, Xiaoxiao and Sun, Jiazheng and Liu, Hengyu and Zhang, Zhi and Qiu, Feixiang and Yang, Zhihao},
  booktitle = {2025 9th International Conference on Smart Grid and Smart Cities (ICSGSC)},
  pages     = {456--461},
  year      = {2025},
  organization = {IEEE}
}

@inproceedings{chen2025reliability,
  title     = {Improving the operational reliability of smart grids: An ultra-short-term load forecasting scheme based on anisotropic radial basis function kernel extreme learning machine},
  author    = {Chen, Kerui and Zhang, Yuhua and Dai, Qiqi and Yu, Haixiong},
  booktitle = {2025 9th International Conference on Smart Grid and Smart Cities (ICSGSC)},
  pages     = {422--427},
  year      = {2025},
  organization = {IEEE}
}

@inproceedings{amilia2022designing,
  title={Designing an Optimized Electric Vehicle Charging Station Infrastructure for Urban Area: A Case study from Indonesia},
  author={Amilia, Nissa and Palinrungi, Zulkifli and Vanany, Iwan and Arief, Mansur},
  booktitle={2022 IEEE 25th International Conference on Intelligent Transportation Systems (ITSC)},
  pages={2812--2817},
  year={2022},
  organization={IEEE}
}

@article{zhang2023optimal,
  title={Optimal planning of flood-resilient electric vehicle charging stations},
  author={Zhang, Qian and Yu, Hao and Zhang, Guohui and Ma, Tianwei},
  journal={Computer-Aided Civil and Infrastructure Engineering},
  volume={38},
  number={4},
  pages={489--507},
  year={2023},
  publisher={Wiley}
}

@article{hussain2020optimal,
  title={Optimal sizing of battery energy storage system in a fast EV charging station considering power outages},
  author={Hussain, Akhtar and Bui, Van-Hai and Kim, Hak-Man},
  journal={IEEE Transactions on Transportation Electrification},
  volume={6},
  number={2},
  pages={453--463},
  year={2020},
  publisher={IEEE}
}

@article{frade2011optimal,
  title={Optimal location of charging stations for electric vehicles in a neighborhood in Lisbon, Portugal},
  author={Frade, In{\^e}s and Ribeiro, Anabela and Gon{\c{c}}alves, Gon{\c{c}}alo and Antunes, Ant{\'o}nio Pais},
  journal={Transportation Research Record},
  volume={2252},
  number={1},
  pages={91--98},
  year={2011},
  publisher={SAGE Publications Sage CA: Los Angeles, CA}
}

@article{hosseini2015refueling,
  title={Refueling-station location problem under uncertainty},
  author={Hosseini, Meysam and MirHassani, Seyyed A},
  journal={Transportation Research Part E: Logistics and Transportation Review},
  volume={84},
  pages={101--116},
  year={2015},
  publisher={Elsevier}
}

@article{yildiz2019urban,
  title={The urban recharging infrastructure design problem with stochastic demands and capacitated charging stations},
  author={Y{\i}ld{\i}z, Bar{\i}{\c{s}} and Olcaytu, Evren and {\c{S}}en, Ahmet},
  journal={Transportation Research Part B: Methodological},
  volume={119},
  pages={22--44},
  year={2019},
  publisher={Elsevier}
}

@article{snyder2005reliability,
  title={Reliability models for facility location: The expected failure cost case},
  author={Snyder, Lawrence V and Daskin, Mark S},
  journal={Transportation Science},
  volume={39},
  number={3},
  pages={400--416},
  year={2005},
  publisher={INFORMS}
}

@book{daskin2013network,
  title={Network and Discrete Location: Models, Algorithms, and Applications},
  author={Daskin, Mark S},
  year={2013},
  edition={2nd},
  publisher={Wiley}
}

\end{document}